\documentclass{article}

 \usepackage[final]{bdl_2018}

\usepackage[utf8]{inputenc} % allow utf-8 input
\usepackage[T1]{fontenc}    % use 8-bit T1 fonts
\usepackage{hyperref}       % hyperlinks
\usepackage{url}            % simple URL typesetting
\usepackage{booktabs}       % professional-quality tables
\usepackage{amsfonts}       % blackboard math symbols
\usepackage{nicefrac}       % compact symbols for 1/2, etc.
\usepackage{microtype}      % microtypography

\usepackage{mathtools}
\usepackage{amsmath}
\usepackage{amssymb}

\usepackage{amsthm}
\newtheorem{theorem}{Theorem}[section]

\usepackage{multirow}
\usepackage{url}

\usepackage{breakurl}
\usepackage{appendix}
\usepackage{makecell}
\usepackage{booktabs}
\newcommand{\bftab}{\fontseries{b}\selectfont}

\title{Bayesian Neural Network Ensembles}

\author{
  Tim ~Pearce \\
  Department of Engineering\\
  University of Cambridge\\
  \texttt{tp424@cam.ac.uk} \\
   \And
   Mohamed Zaki \\
  Department of Engineering\\
  University of Cambridge\\
   \texttt{ } \\
      \And
   Andy Neely \\
  Department of Engineering\\
  University of Cambridge\\
   \texttt{ } \\
}

\usepackage{makeidx}
\makeindex
\printindex

\begin{document}

\maketitle

%\begin{abstract}
%
%Ensembling neural networks (NNs) provides an easily implementable, scalable method for uncertainty quantification that has worked well empirically. Its drawback is that it departs from the trusted Bayesian framework. There is therefore interest in alligning these two paradigms. In this work we derive one modification to the usual ensembling process that does result in Bayesian behaviour for wide NNs - regularising parameters about values drawn from a prior distribution.

%Ensembles of neural networks (NNs) have long been used to estimate predictive uncertainty; a small number of NNs are trained from different initialisations and sometimes on differing versions of the dataset. The variance of the ensemble's predictions is interpreted as its epistemic uncertainty. The appeal of ensembling stems from being a collection of regular NNs - this makes them both scalable and easily implementable. They have achieved strong empirical results in recent years, often presented as a practical alternative to more costly Bayesian NNs (BNNs). The departure from Bayesian methodology is of concern since the Bayesian framework provides a principled, widely-accepted approach to handling uncertainty.

%In this extended abstract we derive and implement a modified NN ensembling scheme, which provides a consistent estimator of the Bayesian posterior in wide NNs - regularising parameters about values drawn from a prior distribution.
%
%\end{abstract}

\section{Introduction}

Ensembles of neural networks (NNs) have long been used to estimate predictive uncertainty 
\citep{Tibshirani1996, Heskes1996}; a small number of NNs are trained from different initialisations and sometimes on differing versions of the dataset. The variance of the ensemble's predictions is interpreted as its epistemic uncertainty. The appeal of ensembling stems from being a collection of regular NNs - this makes them both scalable and easily implementable.

NN ensembles have continued to achieve strong empirical results in recent years, for example in Lakshminarayanan et al. (\citeyear{Lakshminarayanan2016}), where it was presented as a practical alternative to more costly Bayesian NNs (BNNs). The departure from Bayesian methodology is of concern since the Bayesian framework provides a principled, widely-accepted approach to handling uncertainty. %Gal (\citeyear{Gal2016}), who provides a recent authoritative text on BNNs, saying ensembling ``\textit{cannot technically be considered as approximate inference in BNNs}'' [p. 27].

Several recent works have explored links between ensembles and Bayesian inference. Variants of an ensembling scheme known to be consistent for Bayesian linear regression have been applied directly to NNs \citep{Lu2017, Osband2017}. In this extended abstract we derive and implement a modified ensembling scheme specifically for NNs, which provides a consistent estimator of the Bayesian posterior in wide NNs - regularising parameters about values drawn from a prior distribution.

%However, the specific scheme required to be consistent for an ensemble of NNs has not yet been derived. In this extended abstract we derive and implement exactly this. We propose a surprisingly simple ensembling scheme that provides a consistent estimator of the Bayesian posterior in wide NNs - regularising parameters about values drawn from a prior distribution.

%It has been shown that a slight modification to the usual ensembling process provides a consistent estimator of the Bayesian posterior in linear regression \citep{Lu2017, Osband2017}. 

% adding noise to the regularisation term of a loss function as well as the targets 

%We include analysis as to why this method is valid for wide NNs.
 
%We draw comparisons to predicitve distribution produced by the equivalent NN kernels for GP, which produces the output of exact inference of (infinitely) wide NNs. Using this as a comparison shows our proposed modification is essential to produce Bayesian behaviour in ensembles, and provides closer to analytical behaviour than either mean-field VI, or MC Dropout.

\section{Randomised MAP Sampling}
\label{sec_rand_map}

Recent work in the Bayesian community, and independently in the reinforcement learning community, has begun to explore an approach to Bayesian inference that will be novel to many readers. Roughly speaking, it exploits the fact that adding a regularisation term to a loss function returns maximum a posteriori (MAP) estimates of parameters with normally distributed priors centred at zero \citep{MacKay1992}. For a regression problem this loss is of the form,

\begin{equation}
\label{eqn_reg_loss_matrix}
Loss_{regularise} =  
\frac{1}{N} \lvert \lvert \mathbf{y} - \hat{\mathbf{y}} \rvert \rvert ^2_2
+ \frac{1}{N} \lvert \lvert \pmb{\Gamma}^{1/2} \pmb{\theta} \rvert \rvert ^2_2,
\end{equation}

where $\mathbf{y}$ is a vector of targets, $\hat{\mathbf{y}}$ is the NN's predictions, $\pmb{\theta}$ is a flattened vector of NN parameters, and $\pmb{\Gamma}$ is a diagonal square regularisation matrix with it's $k^{th}$ diagonal element representing the ratio of data noise variance to prior variance for parameter $\theta_k$. Data noise is assumed normally distributed and homoskedastic in this work.

Injecting noise into this loss, either to targets or regularisation term, and sampling repeatedly (i.e. ensembling), produces a \textit{distribution} of MAP solutions which can approximate the true posterior. This can be an efficient method to sample from high-dimensional posteriors \citep{Bardsley2018}.

%\citep{Gu2007, Chen2012, Bardsley2018}

Whilst it is straightforward to select the noise distribution that produces exact inference in linear regression models, there is difficulty in transferring this idea to NNs. Directly applying the noise distribution from the linear case to NNs has had some empirical success, despite not reproducing the true posterior \citep{Lu2017, Osband2018}. A more accurate, though more computationally demanding solution, is to wrap the optimisation step into an MCMC procedure \citep{Bardsley2012, Bardsley2018}. We name this family of schemes \textit{randomised MAP sampling}.

\subsection{Normally Distributed Prior and Likelihood}
We consider randomised MAP sampling for the case of multivariate normal prior and (normalised) likelihood, $\mathcal{N}(\pmb{\mu}_{prior},\pmb{\Sigma}_{prior})$, $\mathcal{N}(\pmb{\mu}_{like},\pmb{\Sigma}_{like})$. The posterior, also multivariate normal, is given by Bayes rule, $\mathcal{N}(\pmb{\mu}_{post},\pmb{\Sigma}_{post}) \propto \mathcal{N}(\pmb{\mu}_{prior},\pmb{\Sigma}_{prior}) \cdot \mathcal{N}(\pmb{\mu}_{like},\pmb{\Sigma}_{like})$. The MAP solution is simply $\pmb{\mu}_{post}$, for which a standard result exists,

\begin{equation}
\label{eq_mu_post_text}
\pmb{\mu}_{post} = (\pmb{\Sigma}^{-1}_{like} +\pmb{\Sigma}^{-1}_{prior})^{-1}(\pmb{\Sigma}^{-1}_{like} \pmb{\mu}_{like} + \pmb{\Sigma}^{-1}_{prior} \pmb{\mu}_{prior} ).
\end{equation}

In randomised MAP sampling we are interested in injecting noise so that $ \mathbb{V}ar[\pmb{\mu}_{post} ]= \pmb{\Sigma}_{post}$. Previous work analysing linear regression found that injecting noise into both $\pmb{\mu}_{prior}$ and $\pmb{\mu}_{like}$ can provide a consistent estimator of the true posterior. However, beyond the linear case this approach fails as manipulation of $\pmb{\mu}_{like}$ via targets, $\mathbf{y}$, is complex and creates conflicts amongst parameters.

If instead $\pmb{\mu}_{prior}$ is chosen as the sole noise source, this problem is avoided. In order to inject this noise, let us replace $\pmb{\mu}_{prior}$ with some noisy random variable, $\pmb{\theta}_{0}$, and denote $\pmb{\mu}_{post}^{MAP}(\pmb{\theta}_0)$ the resulting MAP estimate,

\begin{equation}
\label{eq_mu_post_text_anch}
\pmb{\mu}_{post}^{MAP} (\pmb{\theta}_{0}) = (\pmb{\Sigma}^{-1}_{like} +\pmb{\Sigma}^{-1}_{prior})^{-1}(\pmb{\Sigma}^{-1}_{like} \pmb{\mu}_{like} + \pmb{\Sigma}^{-1}_{prior} \pmb{\theta}_{0} ),
\end{equation}

which could be found in practise by minimisation of a slightly modified `anchored' loss function,

\begin{equation}
\label{eqn_anch_loss_matrix}
Loss_{anchor} =  
\frac{1}{N} \lvert \lvert \mathbf{y} - \hat{\mathbf{y}} \rvert \rvert ^2_2
+ \frac{1}{N} \lvert \lvert \pmb{\Gamma}^{1/2} (\pmb{\theta} - \pmb{\theta}_{0}) \rvert \rvert ^2_2.
\end{equation}

Derivation of the noise distribution required for $\pmb{\theta}_{0}$ is found from eq. \ref{eq_mu_post_text_anch}, setting $\mathbb{E}[ \pmb{\mu}_{post}^{MAP}(\pmb{\theta}_0) ]= \pmb{\mu}_{post}$ and $\mathbb{V}ar[ \pmb{\mu}_{post}^{MAP}(\pmb{\theta}_0) ]= \pmb{\Sigma}_{post}$. We find $\pmb{\theta}_{0} \sim \mathcal{N}(\pmb{\mu}_{0},\pmb{\Sigma}_{0})$ with, 

\begin{equation}
\pmb{\mu}_{0} = \pmb{\mu}_{prior},    \: \: \: \:
\pmb{\Sigma}_{0} = \pmb{\Sigma}_{prior} +  \pmb{\Sigma}_{prior}^2 \pmb{\Sigma}_{like}^{-1}.
\label{eq_anch_full}
\end{equation}

\section{Application to NNs}
\label{sec_}

Although the previous section's result is of interest, evaluating eq. \ref{eq_anch_full} requires knowing the likelihood covariance, $\pmb{\Sigma}_{like}$. Estimating this for a NN is far from simple: NNs are unidentifiable, their likelihood variances and correlations vary greatly across parameters, and shift during training. This impasse can be solved in a surprising way. From eq. \ref{eq_anch_full} we see that $diag(\pmb{\Sigma}_{0}) \geq diag(\pmb{\Sigma}_{prior})$. In fact, with increasing NN width, $H$, the term $\pmb{\Sigma}_{prior}^2 \pmb{\Sigma}_{like}^{-1}$ tends to a zero matrix (see appendix for proof). Therefore choosing this lower bound and setting $\pmb{\Sigma}_{0} = \pmb{\Sigma}_{prior}$ is valid for wide NNs.

%To see why $\pmb{\Sigma}_{prior}^2 \pmb{\Sigma}_{like}^{-1}$ tends to a zero matrix with increasing $H$, first we consider priors. It is usual to scale prior covariance in BNNs according to $1/H$ \citep{Neal1997}. This means the term of interest, $\pmb{\Sigma}_{prior}^2 \pmb{\Sigma}_{like}^{-1} \propto \frac{1}{H^2} \pmb{\Sigma}_{like}^{-1}$, which clearly decreases with $H$. Secondly, increasing $H$ creates more parameters and hence a higher probability of strong correlations amongst them. In a similar way to multicollinearity in linear regression, this has the effect of increasing the magnitude of $\pmb{\Sigma}_{like}$. See also Cheng et al. (\citeyear{Cheng2018}). Hence $\pmb{\Sigma}_{like}^{-1}$ decreases. Both these results suggest, $\lim_{H\to \infty} \pmb{\Sigma}_{prior}^2 \pmb{\Sigma}_{like}^{-1} \to 0$.

Stepping back, we note $\pmb{\Sigma}_{0} \approx \pmb{\Sigma}_{prior}$ is only true in the case the posterior is dominated by the prior distribution rather than the likelihood. This occurs in BNNs because the role of priors is slightly abused as a source of regularisation in an over-paramatised model. This observation is significant as it allows us to \textbf{relax our assumption that the prior and likelihood be normally distributed}. Instead, we can say that our method is valid provided the posterior is dominated by the prior.

A surprisingly simple result remains: a wide NN minimising the loss function in eq. \ref{eqn_anch_loss_matrix}, and with $\pmb{\theta}_{0} \sim \mathcal{N}(\pmb{\mu}_{prior},\pmb{\Sigma}_{prior})$, provides a consistent estimator of the posterior.

\subsection{Number in the Ensemble}

If each NN is a single posterior sample, it might be expected that an inordinate number are required to capture the true posterior parameter distributions. But the parameter distributions themselves are of little interest in the context of a NN, it is the predictive distribution that is of sole interest. In this way we \textbf{move from doing inference in parameter space to output space}. Given that each NN provides an independent sample from a posterior predictive distribution, a relatively small number of NNs can give a good approximation. An ensemble size of 5-10 worked well in experiments. This number does not increase with dimensionality of input or output.

%More will provide better accuracy, but at increased computational cost. The marginal accuracy improvement per NN decreases in the usual manner of variance and sample size - for a Gaussian predictive distribution, the standard error of mean and variance decrease according to $1/M$, where $M$ is ensemble size. As a rule of thumb, we suggest an ensemble size of 5-10. This number does not increase with dimensionality of input or output.

\section{Experiments} 
\label{sec_}

%Figure \ref{fig_methods} shows predictive distributions for a toy regression problem for wide single-layer NNs. We compare our method to popular Bayesian inference methods (columns); GP refers 

In figure \ref{fig_methods} we compare predictive distributions produced by popular Bayesian inference methods in wide (100 node) single-layer NNs, with our method, on a toy regression problem. We used several non-linearities for which analytical GP kernels exist - ReLU, ERF (sigmoidal) and RBF. GP and HMC produce `gold standard' Bayesian inference, and we judge the remaining methods, which are scalable approximations, to them. `VI' denotes mean-field variational inference with Gaussian approximating distributions. `MC Dropout' refers to the popular method proposed in Gal \& Ghahramani (\citeyear{Gal2015}). `Our method' implements the scheme described in this work, with ten NNs per ensemble.

The predictive distributions produced by our method appear good, if slightly wavy, approximations of gold standard inference. However, there does appear to be a tendency to over predict the variance. It captures uncertainty in interpolated regions significantly better than VI and MC Dropout, neither of which account for correlations between parameters.

These plots illustrate one more important point. An example of when ensembling fails to perform Bayesian inference was provided by Gal (\citeyear{Gal2016}) [p. 27]: an ensemble of RBF NNs would output zero with high confidence when predicting far from the training data, and this would not be the case for the equivalent RBF GP which was the squared exponential (SE) kernel. However, the RBF GP is \textit{not} the SE kernel except in the special case of infinite variance priors \citep{Williams1996}. Figure \ref{fig_methods}, bottom left, shows an actual RBF GP with finite variance. In fact the GP outputs zero with high confidence far from the data, as do all methods.

Table \ref{tab_regression} (appendix) gives results of our method on ten standard benchmarking datasets \citep{Hernandez-Lobato2015}. Our method outperforms Deep Ensembles \citep{Lakshminarayanan2016} on datasets where the primary source of uncertainty was epistemic. Code of our implementation is available at \url{https://github.com/TeaPearce}.

%VI captures the geometry of extrapolations reasonably well, but because it does not account for correlations between parameters, fails to capture uncertainty in interpolations..

\begin{figure*}[!t]
%\vskip 0.1in
%\vskip 0.2in
\begin{center}
%	\begin{minipage}{1\textwidth}
%	\centering
%	 {\large Inference Methods}
%	 \vspace{0.1in}
%	\end{minipage}

    \begin{minipage}{0.205\textwidth}
        \centering
        \includegraphics[width=0.93\textwidth]{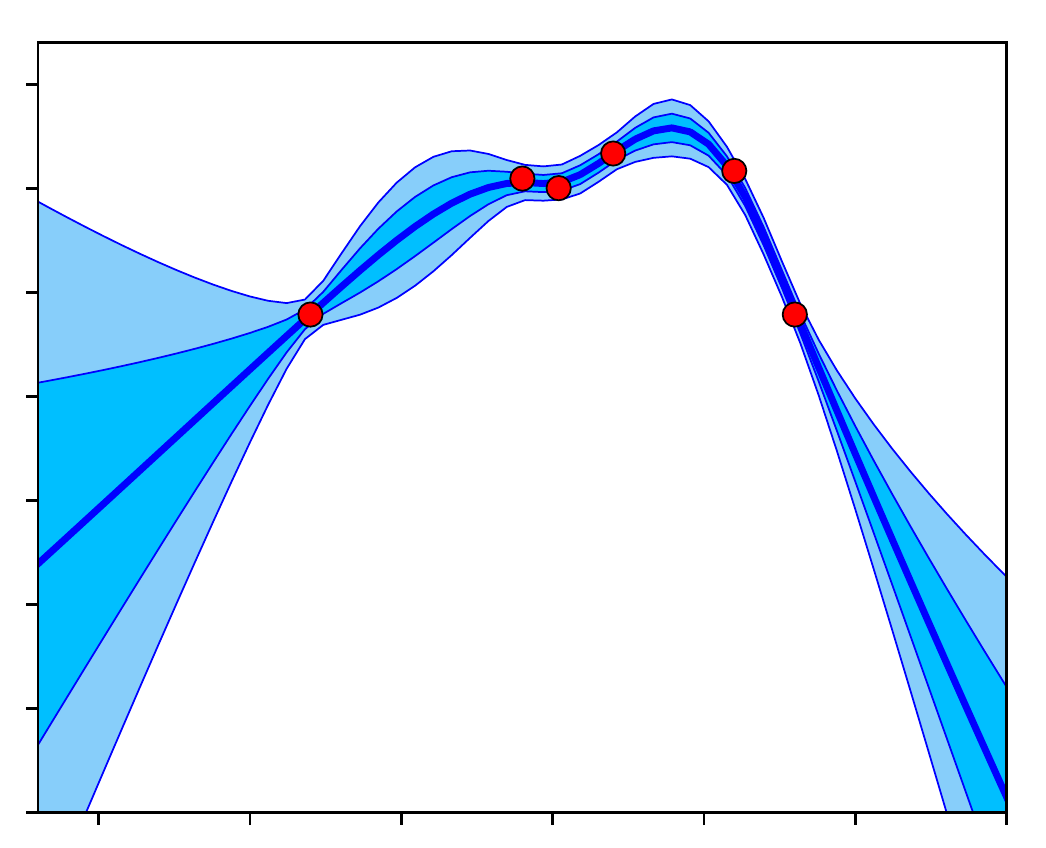}
        %\put (-87,80) {\small Gaussian Process}
        \put (-75,63) {\small Ground Truth - GP}
        %\put (-105,80) {\small Ground Truth, \tiny{ReLU GP}}
       % \put (-82,8) {\small ReLU Activations}
        \put(-83,20){\rotatebox{90}{\small ReLU}}
    \end{minipage}
     \hspace{-0.15in}
    \begin{minipage}{0.205\textwidth}
        \centering
        \includegraphics[width=0.93\textwidth]{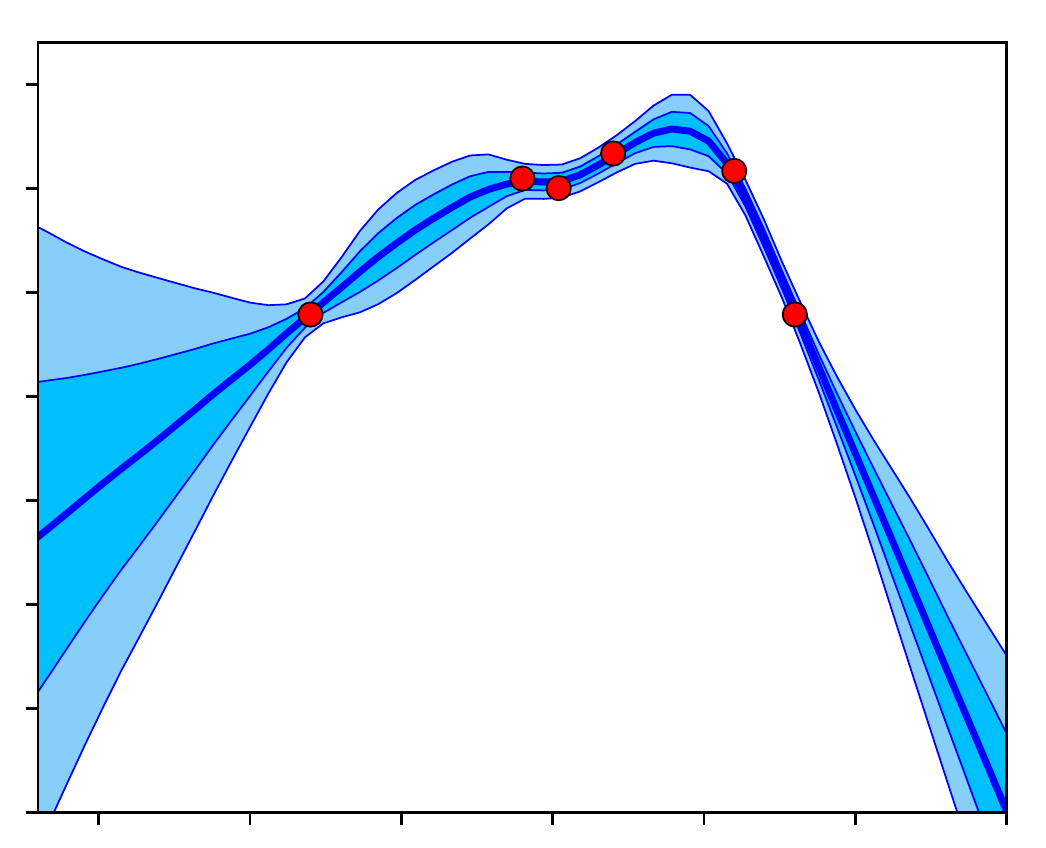}
        \put (-70,63) {\small Hamiltonian MC}
    \end{minipage}
    \hspace{-0.15in}
    \begin{minipage}{0.205\textwidth}
        \centering
        \includegraphics[width=0.93\textwidth]{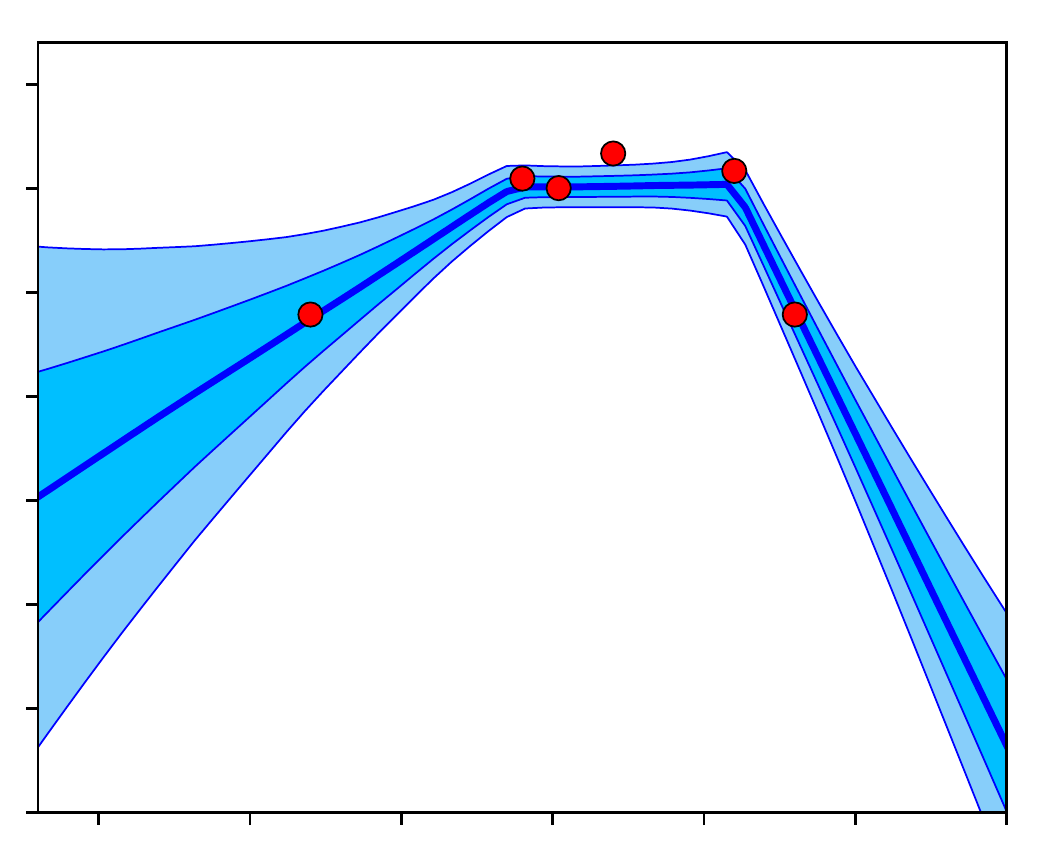}
        \put (-45,63) {\small VI}
    \end{minipage}
    \hspace{-0.15in}
    \begin{minipage}{0.205\textwidth}
        \centering
        \includegraphics[width=0.93\textwidth]{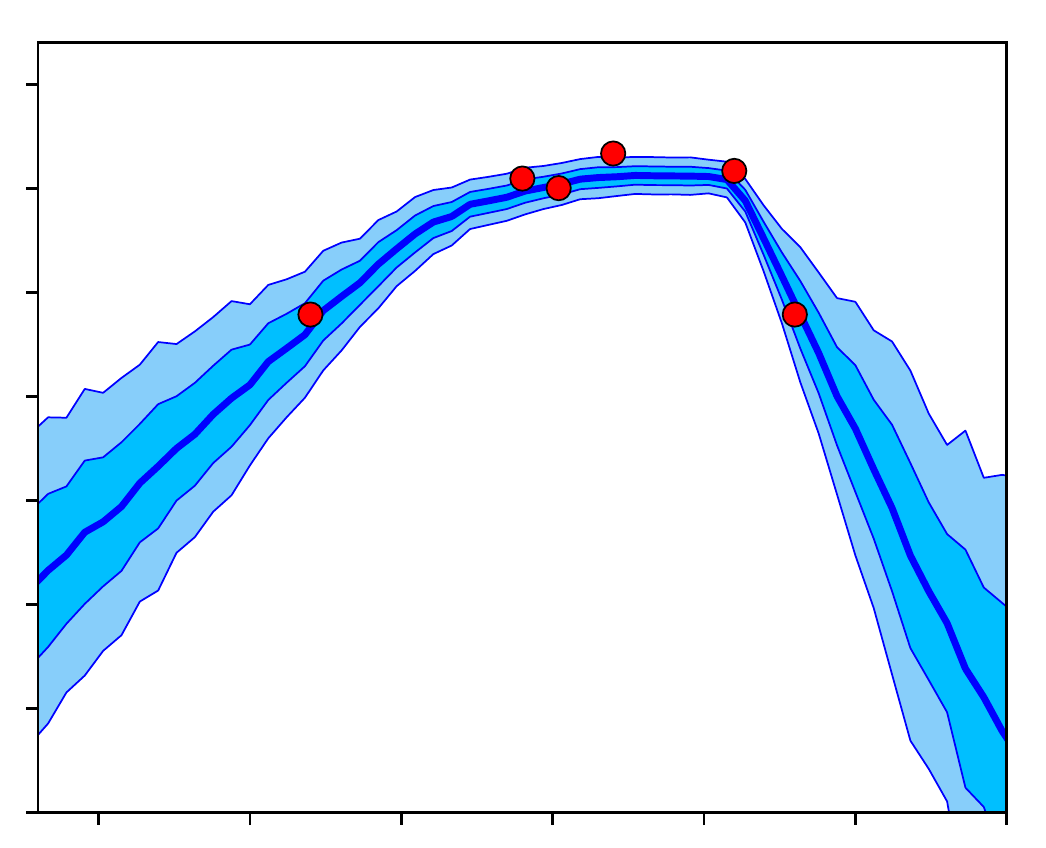}
        \put (-65,63) {\small MC Dropout}
    \end{minipage}
    \hspace{-0.15in}
    \begin{minipage}{0.205\textwidth}
        \centering
        \includegraphics[width=0.93\textwidth]{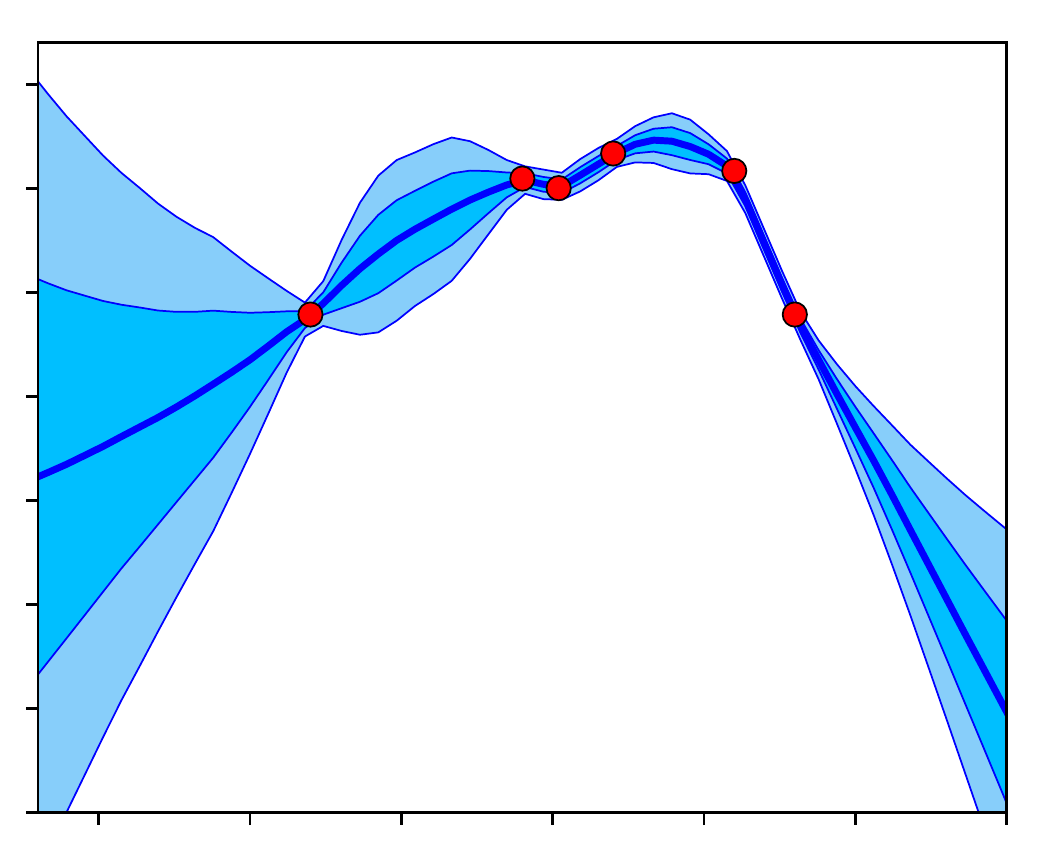}
        \put (-65,63) {\small Our Method}
%        \put (-94,100) {\small (Our Method)}
    \end{minipage}
    
   \vspace{-0.05in}
   
    \begin{minipage}{0.205\textwidth}
        \centering
        \includegraphics[width=0.93\textwidth]{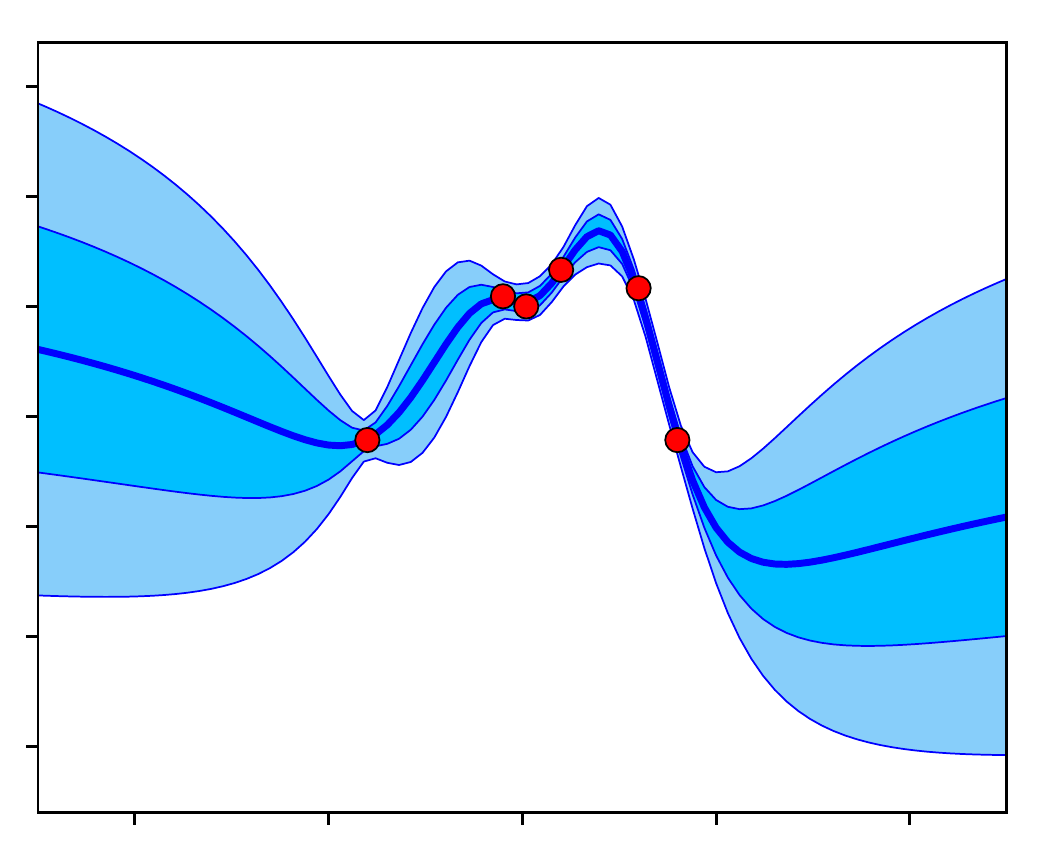}
        %\put (-85,8) {\small Sigmoidal (ERF)}
        \put(-83,12){\rotatebox{90}{\small Sigmoidal}}
    \end{minipage}
    \hspace{-0.15in}
    \begin{minipage}{0.205\textwidth}
        \centering
        \includegraphics[width=0.93\textwidth]{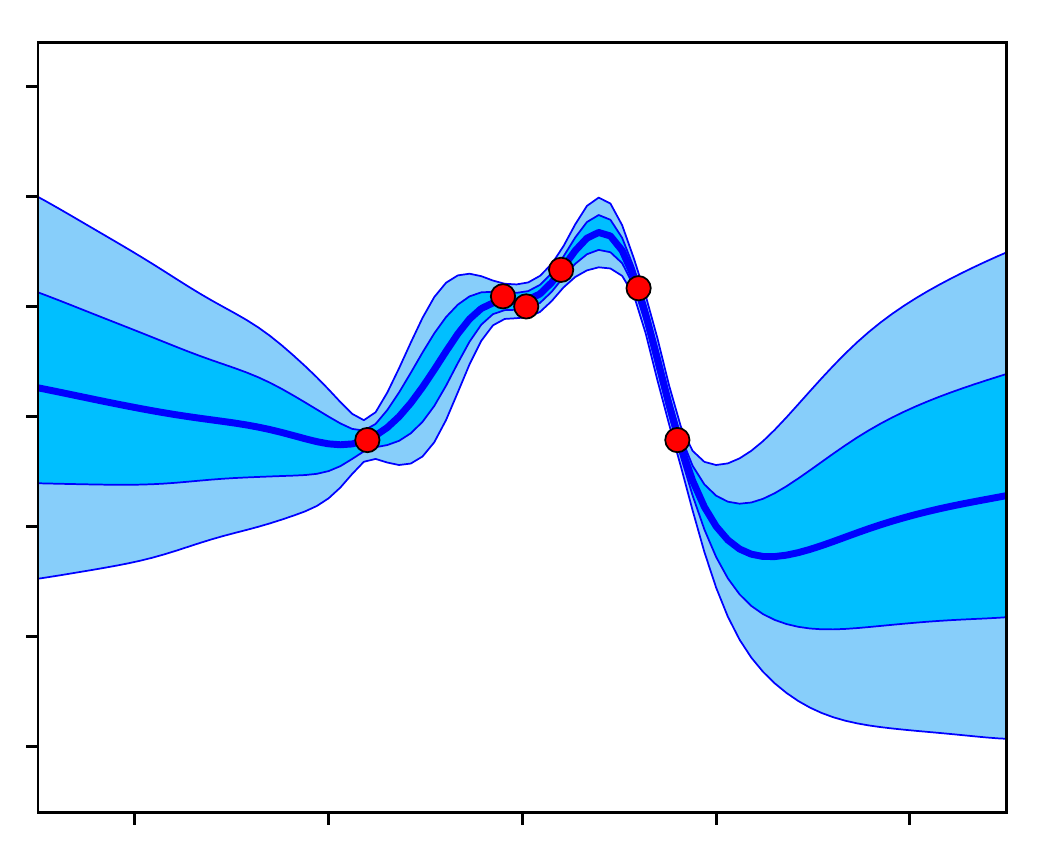}
    \end{minipage}
    \hspace{-0.15in}
    \begin{minipage}{0.205\textwidth}
        \centering
        \includegraphics[width=0.93\textwidth]{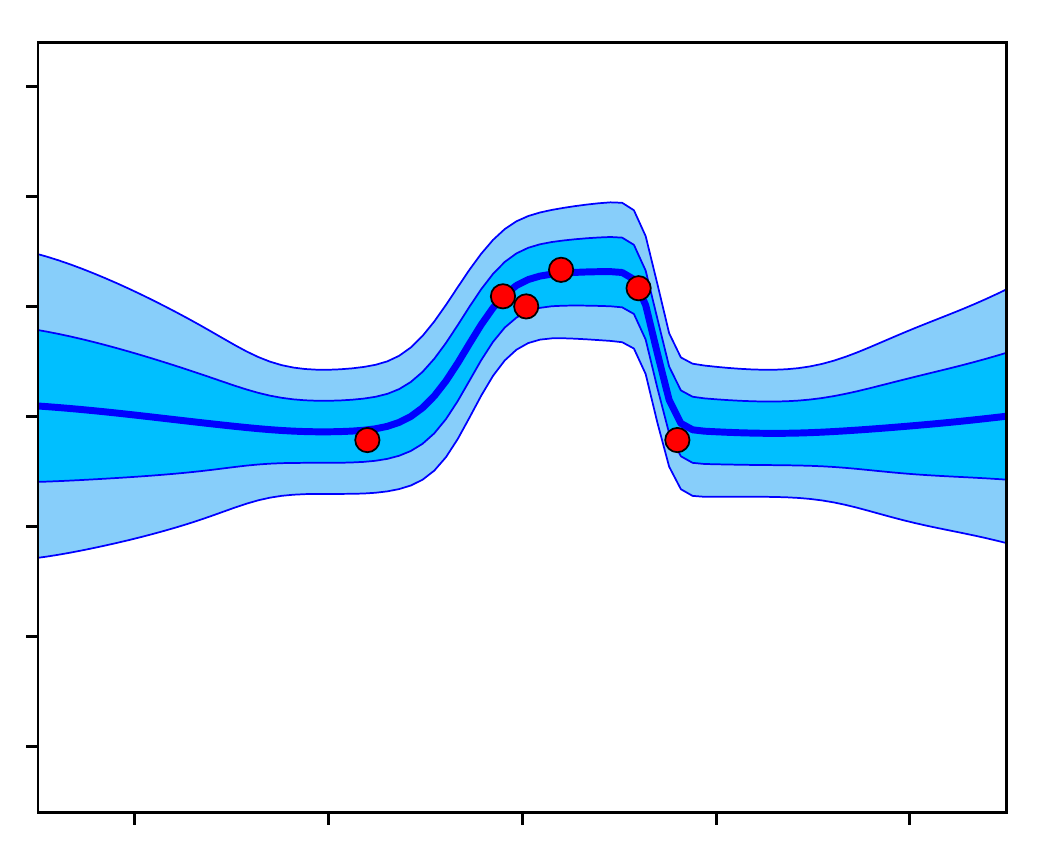}
    \end{minipage}
    \hspace{-0.15in}
    \begin{minipage}{0.205\textwidth}
        \centering
        \includegraphics[width=0.93\textwidth]{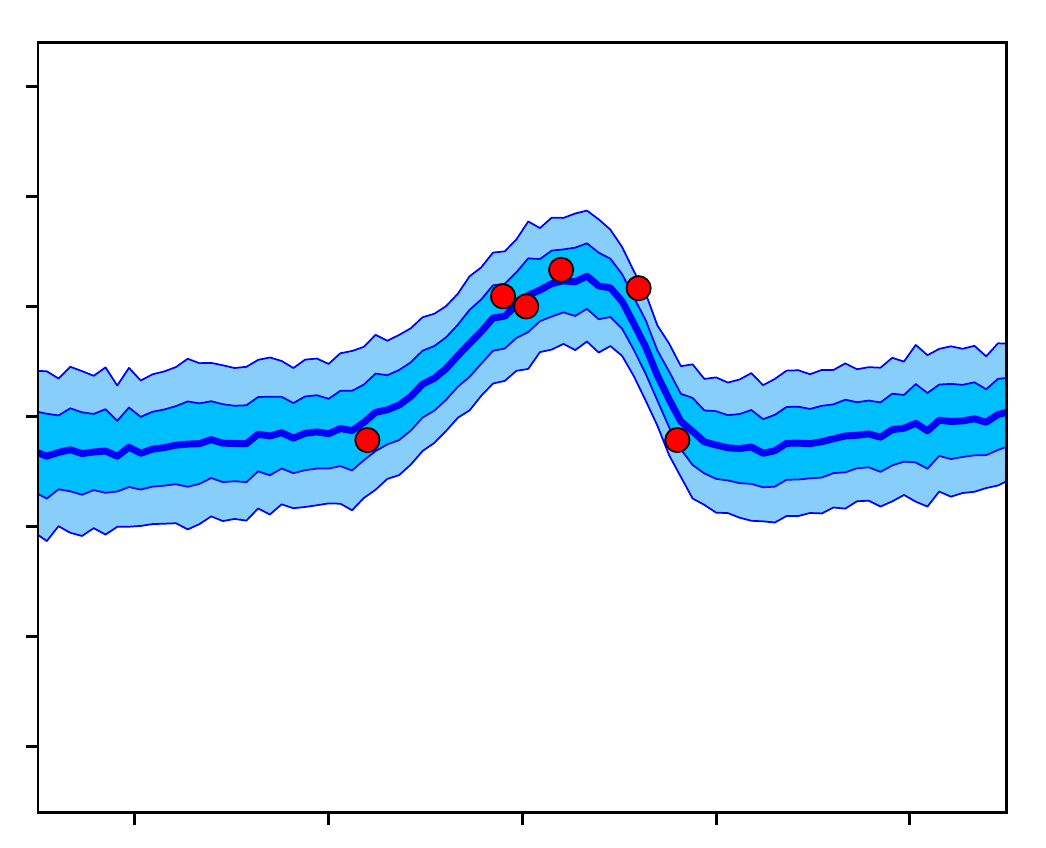}
    \end{minipage}
    \hspace{-0.15in}
    \begin{minipage}{0.205\textwidth}
        \centering
        \includegraphics[width=0.93\textwidth]{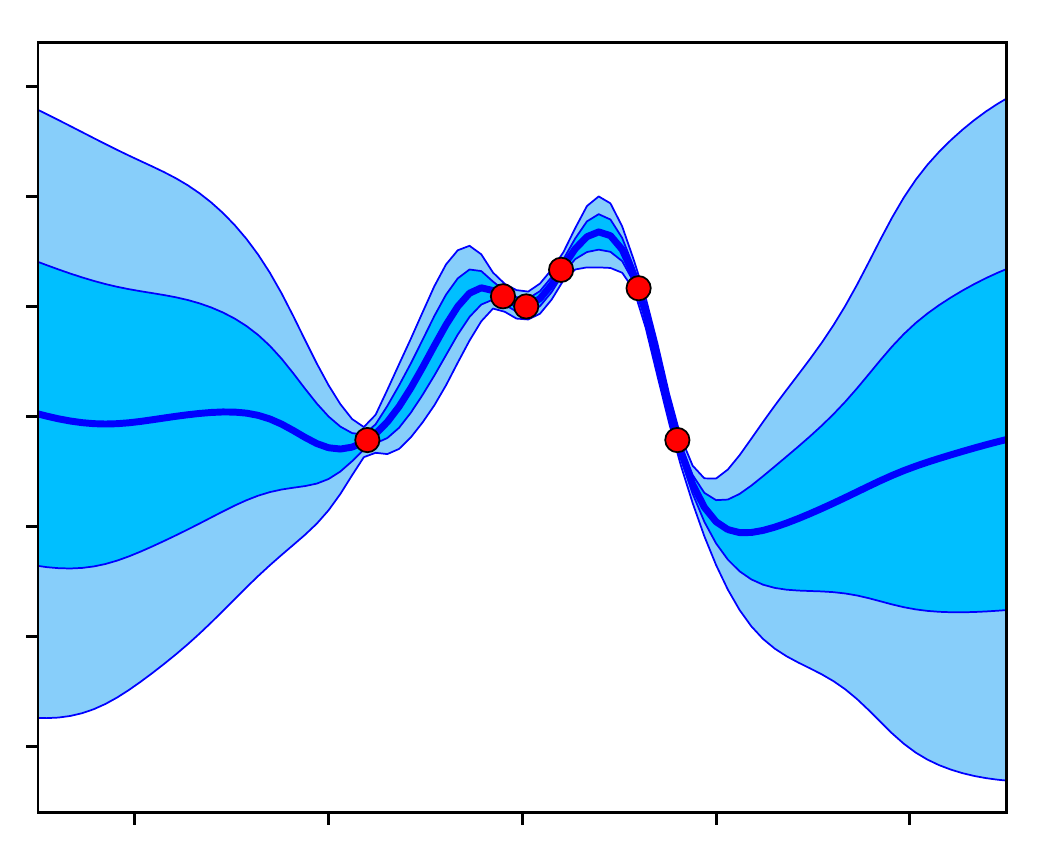}
    \end{minipage}
    
       \vspace{-0.05in}
       
    \begin{minipage}{0.205\textwidth}
        \centering
        \includegraphics[width=0.93\textwidth]{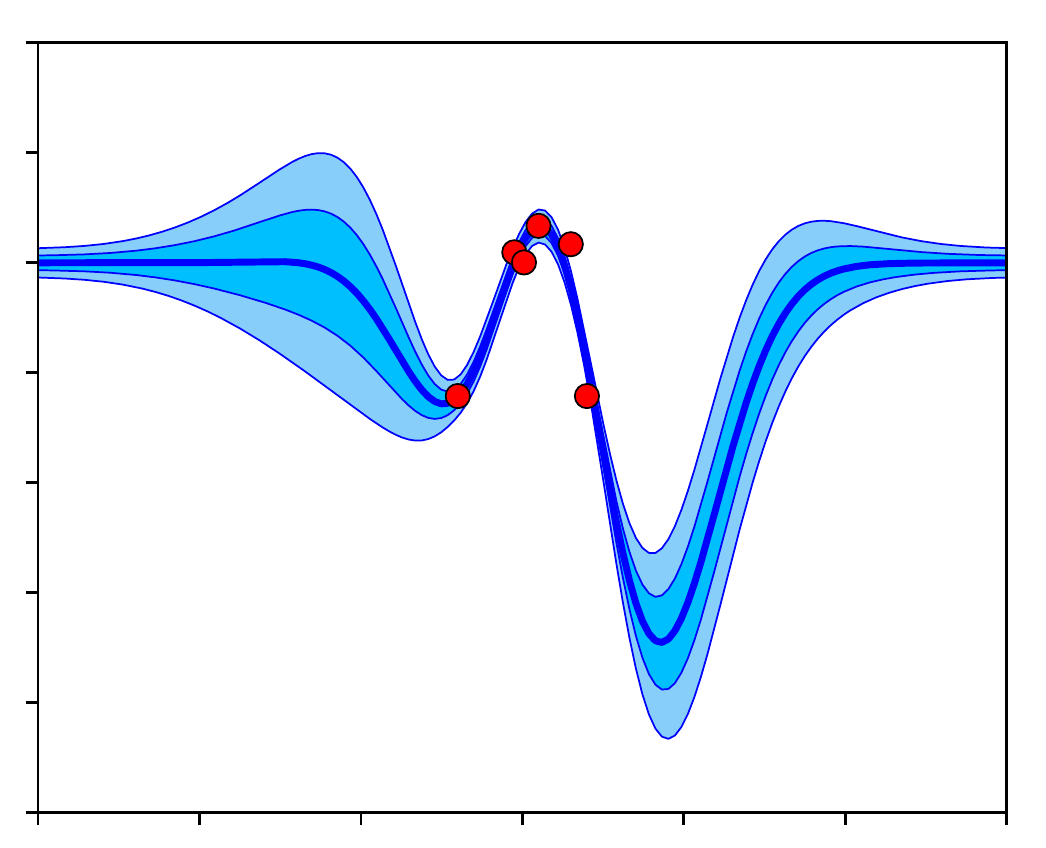}
        %\put (-85,8) {\small RBF}
        \put(-83,25){\rotatebox{90}{\small RBF}}
    \end{minipage}
    \hspace{-0.15in}
    \begin{minipage}{0.205\textwidth}
        \centering
        \includegraphics[width=0.93\textwidth]{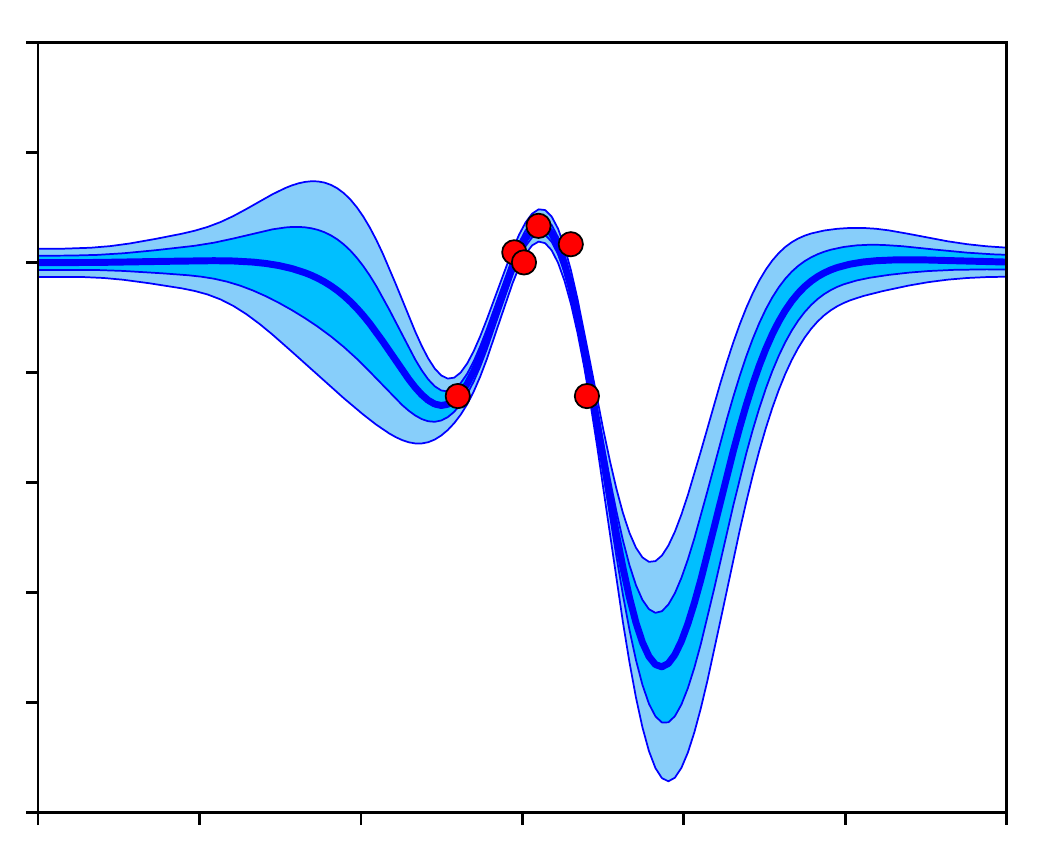}
       % \put (-85,8) {\small Hamiltonian MC}
    \end{minipage}
    \hspace{-0.15in}
    \begin{minipage}{0.205\textwidth}
        \centering
        \includegraphics[width=0.93\textwidth]{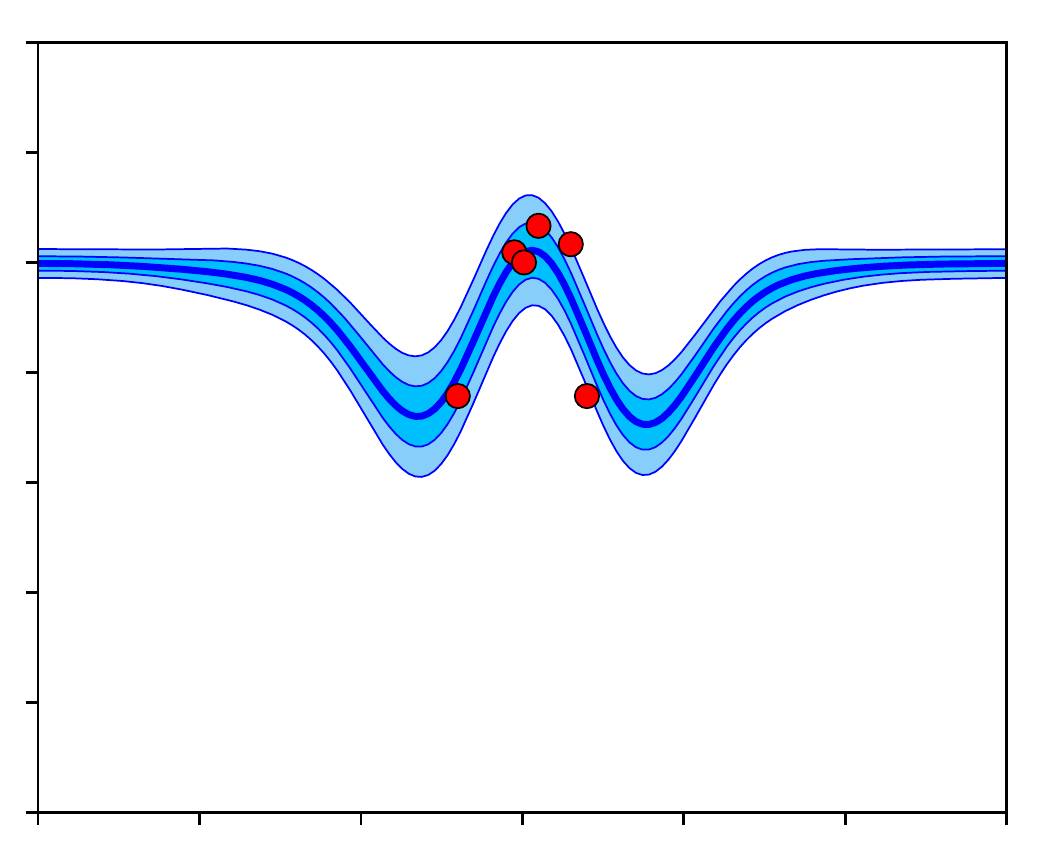}
        %\put (-85,8) {\small Variational Inf.}
    \end{minipage}
    \hspace{-0.15in}
    \begin{minipage}{0.205\textwidth}
        \centering
        \includegraphics[width=0.93\textwidth]{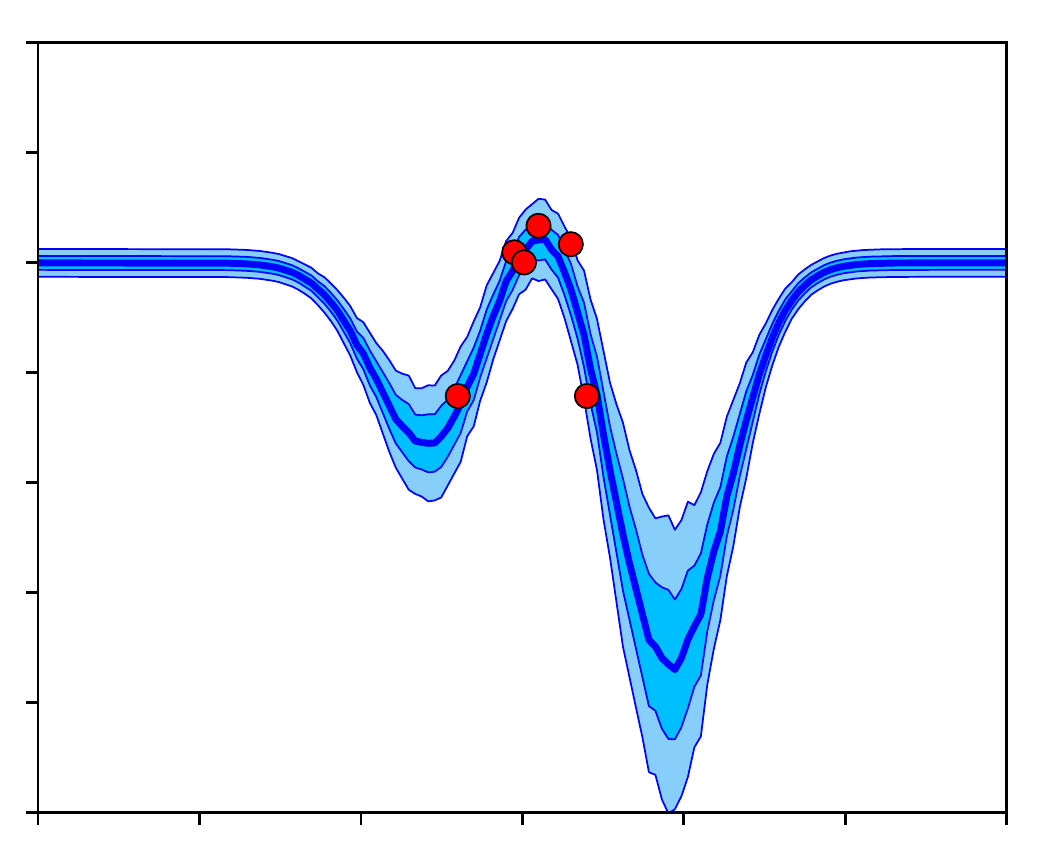}
        %\put (-85,8) {\small MC Dropout}
    \end{minipage}
    \hspace{-0.15in}
    \begin{minipage}{0.205\textwidth}
        \centering
        \includegraphics[width=0.93\textwidth]{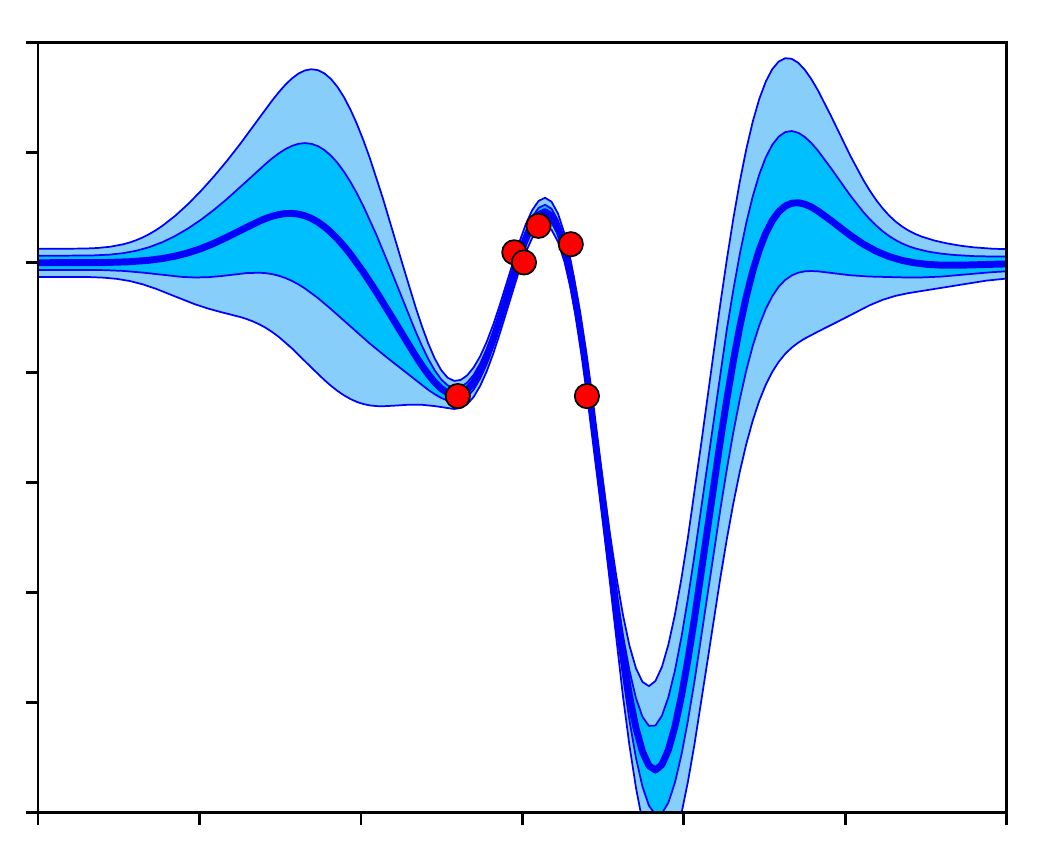}
       % \put (-85,8) {\small Anchored Ensemble}
    \end{minipage}       
       
   %\vspace{0.2in}
\caption{Predictive distributions produced by various inference methods (columns) for various activation functions (rows), e.g. bottom right is a RBF NN with inference by our method.}

%All plots share the same six data points in a 1-D regression problem. GP and HMC provide exact inference methods, the remainder are practical approximations.}
\label{fig_methods}
\end{center}
\vskip -0.2in
\end{figure*}

\section{Conclusion}

This paper considered a method to produce Bayesian behaviour in NN ensembles by leveraging randomised MAP sampling. It departs only slightly from the usual handling of NNs, with parameters regularised around values drawn from a prior distribution. We showed that for NNs of sufficient width, each produces a sample from the posterior predictive distribution. Qualitative and benchmarking experiments were encouraging. Our ongoing work considers extending the presented theory to classification tasks as well as other architectures such as convolutional NNs.

\subsubsection*{Acknowledgments}

The authors thank EPSRC for funding (EP/N509620/1), the Alan Turing Institute for accommodating the lead author during his work (TU/D/000016), and Microsoft for Azure credits. Personal thanks to Nicolas Anastassacos, Ayman Boustati and Ian Osband.

\small
\bibliography{library} % filename of bibtex goes in here

\begin{thebibliography}{}

\bibitem[Bardsley, 2012]{Bardsley2012}
Bardsley, J.~M. (2012).
\newblock {MCMC-based image reconstruction with uncertainty quantification}.
\newblock {\em SIAM Journal on Scientific Computing}, 34(3):1316--1332.

\bibitem[Bardsley et~al., 2014]{Bardsley2018}
Bardsley, J.~M., Solonen, A., Haario, H., and Laine, M. (2014).
\newblock {Randomize-Then-Optimize: A Method for Sampling from Posterior
  Distributions in Nonlinear Inverse Problems}.
\newblock {\em SIAM Journal on Scientific Computing}, 36(4).

\bibitem[Cheng et~al., 2018]{Cheng2018}
Cheng, X., Khomtchouk, B., Matloff, N., and Mohanty, P. (2018).
\newblock {Polynomial Regression As an Alternative to Neural Nets}.
\newblock pages 1--28.

\bibitem[Gal, 2016]{Gal2016}
Gal, Y. (2016).
\newblock {\em {Uncertainty in Deep Learning}}.
\newblock PhD thesis.

\bibitem[Gal and Ghahramani, 2015]{Gal2015}
Gal, Y. and Ghahramani, Z. (2015).
\newblock {Dropout as a Bayesian Approximation: Representing Model Uncertainty
  in Deep Learning}.
\newblock In {\em Proceedings of the 33rd International Conference on Machine
  Learning}.

\bibitem[Hern{\'{a}}ndez-Lobato and Adams, 2015]{Hernandez-Lobato2015}
Hern{\'{a}}ndez-Lobato, J.~M. and Adams, R.~P. (2015).
\newblock {Probabilistic Backpropagation for Scalable Learning of Bayesian
  Neural Networks}.
\newblock In {\em Proceedings of the 32nd International Conference on Machine
  Learning}.

\bibitem[Heskes, 1996]{Heskes1996}
Heskes, T. (1996).
\newblock {Practical confidence and prediction intervals}.
\newblock In {\em Advances in Neural Information Processing Systems 9}.

\bibitem[Lakshminarayanan et~al., 2017]{Lakshminarayanan2016}
Lakshminarayanan, B., Pritzel, A., and Blundell, C. (2017).
\newblock {Simple and Scalable Predictive Uncertainty Estimation using Deep
  Ensembles}.
\newblock In {\em 31st Conference on Neural Information Processing Systems}.

\bibitem[Lu and {Van Roy}, 2017]{Lu2017}
Lu, X. and {Van Roy}, B. (2017).
\newblock {Ensemble Sampling}.
\newblock In {\em 31st Conference on Neural Information Processing Systems}.

\bibitem[MacKay, 1992]{MacKay1992}
MacKay, D. J.~C. (1992).
\newblock {A Practical Bayesian Framework for Backpropagation Networks}.
\newblock {\em Neural Computation}, 4(3):448--472.

\bibitem[Neal, 1997]{Neal1997}
Neal, R.~M. (1997).
\newblock {\em {Bayesian Learning for Neural Networks}}.
\newblock PhD thesis.

\bibitem[Osband et~al., 2018]{Osband2018}
Osband, I., Aslanides, J., and Cassirer, A. (2018).
\newblock {Randomized Prior Functions for Deep Reinforcement Learning}.
\newblock In {\em 32nd Conference on Neural Information Processing Systems
  (NIPS 2018)}.

\bibitem[Osband et~al., 2017]{Osband2017}
Osband, I., Russo, D., Wen, Z., and {Van Roy}, B. (2017).
\newblock {Deep Exploration via Randomized Value Functions}.

\bibitem[Tibshirani, 1996]{Tibshirani1996}
Tibshirani, R. (1996).
\newblock {A Comparison of Some Error Estimates for Neural Network Models}.
\newblock {\em Neural Computation}, 8:152--163.

\bibitem[Williams, 1996]{Williams1996}
Williams, C. K.~I. (1996).
\newblock {Computing with infinite networks}.
\newblock In {\em Advances in Neural Information Processing Systems 9}.

\end{thebibliography}
\bibliographystyle{apalike}

\normalsize  
\newpage
\begin{appendices}
\section{Appendix}

\subsection{Proofs}

\begin{theorem}
$\pmb{\Sigma}_{prior}^2 \pmb{\Sigma}_{like}^{-1}$ tends to a zero matrix with increasing $H$.
\end{theorem}
 
\begin{proof}
First we consider priors, $\pmb{\Sigma}_{prior}$. It is usual to scale prior covariance in BNNs according to $1/H$ \citep{Neal1997}. This means the term of interest, $\pmb{\Sigma}_{prior}^2 \pmb{\Sigma}_{like}^{-1} \propto \frac{1}{H^2} \pmb{\Sigma}_{like}^{-1}$, which clearly decreases with $H$. 

Secondly, increasing $H$ creates more parameters and hence a higher probability of strong correlations amongst them - a phenomenon known as multicollinearity. This has the effect of increasing the magnitude of $\pmb{\Sigma}_{like}$ (see also Cheng et al. (\citeyear{Cheng2018})). Hence $\pmb{\Sigma}_{like}^{-1}$ decreases. 

Both these results suggest, $\lim_{H\to \infty} \pmb{\Sigma}_{prior}^2 \pmb{\Sigma}_{like}^{-1} \to 0$.
\end{proof}

\subsection{Benchmark Results}
\begin{table}[h]%
\caption{Regression benchmark results for a Bayesian ensemble of five NNs}
\begin{center}
\resizebox{1.\columnwidth}{!}{
\begin{tabular}{ l rrr  rrr  rrr}

\Xhline{3\arrayrulewidth}

\multicolumn{1}{c}{}  & \multicolumn{1}{c}{}  & \multicolumn{1}{c}{}  & \multicolumn{1}{c}{}  & \multicolumn{3}{c}{\textbf{RMSE}} & \multicolumn{3}{c}{\textbf{NLL}} \\

\multicolumn{1}{c}{}  & \multicolumn{1}{c}{}  & \multicolumn{1}{c}{}  & \multicolumn{1}{c}{}  &  \multicolumn{1}{c}{Deep Ens.} &  \multicolumn{1}{c}{Bay. Ens.} &  \multicolumn{1}{c}{GP$^1$} &  \multicolumn{1}{c}{Deep Ens.}  &  \multicolumn{1}{c}{Bay. Ens.} &  \multicolumn{1}{c}{GP$^1$} \\ 

\multicolumn{1}{c}{}  & \multicolumn{1}{r}{$N$}  & \multicolumn{1}{r}{$D$}  & \multicolumn{1}{r}{$\hat{\sigma}^2_{\epsilon}$} &  \multicolumn{1}{c}{\small{\textit{State-Of-Art}}}  &  \multicolumn{1}{c}{\small{\textit{Our Method}}} &  \multicolumn{1}{c}{\small{\textit{Gold Standard}}}  &  \multicolumn{1}{c}{\small{\textit{State-Of-Art}}}  &  \multicolumn{1}{c}{\small{\textit{Our Method}}} &  \multicolumn{1}{c}{\small{\textit{Gold Standard}}}  \\ 

\hline 
\multicolumn{10}{c}{High Epistemic Uncertainty}\\
Energy & 768 & 8 & 1e-7& {2.09 $\pm$ 0.29}  & \bftab{0.35 $\pm$ 0.01} & {0.60 $\pm$ 0.02}& {1.38 $\pm$ 0.22} & \bftab{0.96 $\pm$ 0.13} & {0.86 $\pm$ 0.02}  \\
Naval & 11,934 & 16 & 1e-7& \bftab{0.00 $\pm$ 0.00} & \bftab{0.00 $\pm$ 0.00} & {0.00 $\pm$ 0.00}& {-5.63 $\pm$ 0.05}  & \bftab{-7.17 $\pm$ 0.03} & {-10.05 $\pm$ 0.02}  \\
Yacht & 308 & 6 & 1e-7 & {1.58 $\pm$ 0.48} & \bftab{0.57 $\pm$ 0.05} & {0.60 $\pm$ 0.08}& {1.18 $\pm$ 0.21}  & \bftab{0.37 $\pm$ 0.08} & {0.49 $\pm$ 0.07}  \\
\hline 
\multicolumn{10}{c}{Medium Epistemic \& Aleatoric Uncertainty}\\
Kin8nm & 8,192 & 8 & 0.02& {0.09 $\pm$ 0.00}  & \bftab{0.07 $\pm$ 0.00} & {0.07 $\pm$ 0.00}& {-1.20 $\pm$ 0.02}  & \bftab{-1.09 $\pm$ 0.01} & {-1.22 $\pm$ 0.01}  \\
Power & 9,568 & 4 & 0.05& {4.11 $\pm$ 0.17} & \bftab{4.07 $\pm$ 0.04} & {3.97 $\pm$ 0.04}& \bftab{2.79 $\pm$ 0.04} & {2.83 $\pm$ 0.01} & {2.80 $\pm$ 0.01}  \\
Concrete & 1,030 & 8 & 0.05& {6.03 $\pm$ 0.58} & \bftab{4.87 $\pm$ 0.11} & {4.88 $\pm$ 0.13}& {3.06 $\pm$ 0.18} & \bftab{2.97 $\pm$ 0.02} & {2.96 $\pm$ 0.02}  \\
Boston & 506 & 13 & 0.08& {3.28 $\pm$ 1.00} & \bftab{3.09 $\pm$ 0.17} & {2.86 $\pm$ 0.16}& \bftab{2.41 $\pm$ 0.25} & {2.52 $\pm$ 0.05} & {2.45 $\pm$ 0.05}  \\
\hline 
\multicolumn{10}{c}{High Aleatoric Uncertainty}\\
Protein & 45,730 & 9 & 0.5 & {4.71 $\pm$ 0.06}  & \bftab{4.36 $\pm$ 0.02} & {*4.34 $\pm$ 0.02}& \bftab{2.83 $\pm$ 0.02}  & {2.89 $\pm$ 0.01} & {*2.88 $\pm$ 0.00}  \\
Wine & 1,599 & 11 & 0.5& \bftab{0.64 $\pm$ 0.04}  & \bftab{0.63 $\pm$ 0.01} & {0.61 $\pm$ 0.01}& \bftab{0.94 $\pm$ 0.12}  & \bftab{0.95 $\pm$ 0.01} & {0.92 $\pm$ 0.01}  \\
Song Year & 515,345 & 90 & 0.7 & {8.89 $\pm$ NA }   & \bftab{8.82 $\pm$ NA } & {**9.01 $\pm$ NA }& \bftab{3.35 $\pm$ NA } & {3.60 $\pm$ NA } & {**3.62 $\pm$ NA }  \\

\Xhline{3\arrayrulewidth} % thicker line

\end{tabular}
}
\end{center}

\tiny{$^1$ For reference only (not a scalable method).} \tiny{* Trained on $10,000$ rows of data.} \tiny{** Trained on $20,000$ rows of data, tested on $5,000$ data points.}
%, tested on full test fold
\label{tab_regression}
\end{table}

\end{appendices}

\end{document}